\documentclass[sigconf]{acmart}

\AtBeginDocument{%
  }

\acmBooktitle{Preprints under review.}

\usepackage{graphicx}
\usepackage{subcaption} 
\usepackage{algorithm}
\usepackage{algpseudocode}
\theoremstyle{definition}
\newtheorem{definition}{Definition}
\newtheorem{theorem}{Theorem}
\theoremstyle{remark}
\newtheorem*{remark}{Remark}
\usepackage{tabularx} 
\begin{document}

\title{On Leveraging Large Language Models for Enhancing Entity Resolution: A Cost-efficient Approach}

\author{Huahang Li}
\orcid{0009-0000-3087-9608}
\affiliation{%
  \institution{The Hong Kong Polytechnic University}
  \country{Hong Kong}
}
\email{hua-hang.li@connect.polyu.hk}

\author{Longyu Feng}
\affiliation{%
  \institution{The Hong Kong Polytechnic University}
  \country{Hong Kong}
}
\email{longyu.feng@polyu.edu.hk}

\author{Shuangyin Li}
\affiliation{%
  \institution{South China Normal University}
  \city{Guangzhou}
  \country{China}}
\email{shuangyinli@scnu.edu.cn}

\author{Fei Hao}
\affiliation{%
  \institution{The Hong Kong Polytechnic University}
  \country{Hong Kong}
}
\email{ffaye.hao@polyu.edu.hk}

\author{Chen Jason Zhang}
\authornote{Chen Jason Zhang is the corresponding author.}
\affiliation{%
  \institution{The Hong Kong Polytechnic University}
  \country{Hong Kong}
  }
 \email{jason-c.zhang@polyu.edu.hk}

\author{Yuanfeng Song}
\affiliation{%
  \institution{Webank Co. Ltd.}
  \city{Shenzhen}
  \country{China}
 }
\email{yfsong@webank.com}

\renewcommand{\shortauthors}{Huahang Li et al.}

\begin{abstract}
Entity resolution, the task of identifying and merging records that refer to the same real-world entity, is crucial in sectors like e-commerce, healthcare, and law enforcement. Large Language Models (LLMs) introduce an innovative approach to this task, capitalizing on their advanced linguistic capabilities and a ``pay-as-you-go'' model that provides significant advantages to those without extensive data science expertise. However, current LLMs are costly due to per-API request billing. Existing methods often either lack quality or become prohibitively expensive at scale. To address these problems, we propose an uncertainty reduction framework using LLMs to improve entity resolution results. We first initialize possible partitions of the entity cluster, refer to the same entity, and define the uncertainty of the result. Then, we reduce the uncertainty by selecting a few valuable matching questions for LLM verification. Upon receiving the answers, we update the probability distribution of the possible partitions. To further reduce costs, we design an efficient algorithm to judiciously select the most valuable matching pairs to query. Additionally, we create error-tolerant techniques to handle LLM mistakes and a dynamic adjustment method to reach truly correct partitions. Experimental results show that our method is efficient and effective, offering promising applications in real-world tasks.
\end{abstract}



\maketitle

\section{Introduction}
Entity resolution refers to identifying and matching records that refer to the same real-world object. It is also known as entity matching, a critical task in data integration, cleaning, and quality improvement. Living in the ear of the web, entity resolution becomes even more critical. Because the vast web expanse brings an overwhelming amount of information, they are often duplicated, fragmented, or represented in myriad forms. For example, websites, online databases, social media platforms, and other online resources frequently contain overlapping and redundant data \cite{getoor2013entity, christophides2019end}. Table~\ref{tabduplicates} shows a web application that serves as an online directory for professionals, such as LinkedIn. However, some records are duplicated and refer to the same person. We aim to identify and merge these records to create a single and comprehensive record of each entity. 

In the early stage of entity resolution methods, most of them rely on manually defined matching rules and statistical and probabilistic-based models that consider various factors to compute match probabilities \cite{fellegi1969theory, blakely2002probabilistic, elmagarmid2006duplicate}. Then, several supervised machine learning systems are proposed to tackle the entity resolution task \cite{winkler2014matching, ebraheem2017deeper, li2020deep, zeakis2023pre}. Modern state-of-the-art methods mainly adopt a Pre-trained Language Model (PLM) to encode the attributes of given entities and compare their semantic similarities with a trainable classifier model. The PLM is trained on a large corpus to obtain a powerful semantic understanding ability, and the classifier model is fine-tuned on a task-specific entity-matching dataset. However, these methods have two significant drawbacks: first, fine-tuning a classifier model requires lots of labeled data, which leads to high labor costs; second, the fine-tuned task-specific models show poor performance on other out-of-domain data.

With the success of recent Large Language models (LLMs), such as OpenAI's GPT-4, Google's Claude3, etc.. They have shown impressive abilities in understanding text semantics and can tackle some reasoning tasks even with human-level performance \cite{zhao2023survey, min2023recent}. These LLMs are trained on more large, vast, and diverse datasets, enabling them to capture intricate linguistic patterns, contextual nuances, and semantic meanings. Previous studies \cite{narayan2022can, peeters2023entity, li2024booster} have shown a novel solution to the entity resolution task by leveraging the capability of LLMs. It is practical and easy to use, as we can ask LLM to identify it via API request. However, since nowadays most LLMs are charged for API requests, simply posing all the matching questions becomes too costly in large-scale datasets, that is, a dataset containing $n$ records may lead up to $\frac{n(n-1)}{2}$ matching pairs. Though there are some open-sourced LLMs, implemented on a local server, may cost even more by requiring large GPU memory and inference computational power. Moreover, to teach the LLMs to better identify the matching pairs, in-context learning is a practical strategy to enhance the model's ability in entity resolution tasks. However, the increased prompt length may even lead to more cost. 

In this study, we investigate the cost-efficiency of leveraging LLMs in entity resolution tasks by selecting which questions are most valuable. We propose an uncertainty reduction framework. Theoretically, we consider every possible partition of the entities as the entity resolution results. Each possible partition is associated with a probability of being correct. Then, we can use the Shannon entropy to estimate the uncertainty of the entity resolution results. After that, we posed matching questions for LLMs to verify and then adjusted the probabilities of possible partitions based on the answers, thus reducing the uncertainty. After several iterations, we finally got the enhanced entity resolution results. The rationale behind our framework is straightforward: Within a fixed system, uncertainty can only be reduced by external energy input; this is where the LLMs play critical roles. In practice, the possible partitions are initialized by traditional similarity-based tools with different thresholds. The probability distribution can be initialized by statistical methods. The problem now is how to select the most valuable question. Note that there are some inherent properties such as ``transitivity'' in entity resolution results; for example, if we already know $``r_1 =  r_2"$, and $``r_2 =  r_3"$, there is no need to ask whether $``r_2 = r_3"$. Also, each record contains different attributes that may vary in length, resulting in different API request costs. Therefore, we have designed a cost-efficient algorithm to select the matching questions. Moreover, LLM sometimes makes mistakes, and to provide an error-tolerant technique, we design a dynamic adjustment approach. The key contributions of this work are summarized as follows:

\begin{itemize}
\item We propose an uncertainty reduction framework for leveraging the LLMs in entity resolution. We formulate how to calculate the uncertainty reduction caused by matching questions and prove that the uncertainty reduction of a set of MQs is equivalent to the joint entropy of their possible answer set.

\item Based on this formulation, we prove the NP-hardness of the matching questions selection problem. We then propose an optimal greedy-based algorithm that efficiently selects matching questions within budget constraints, effectively reducing uncertainty.

\item Our approach incorporates an error-tolerant design, allowing LLMs to utilize even imperfect answers to decrease uncertainty. We can achieve accurate entity partitions through dynamic adjustments, even if they deviate from initial estimates.

\end{itemize}

\section{Preliminary Knowledge}

In this section, we provide definitions related to the problem we are working on in this paper.

\begin{definition}[Possible Partition]
Let $P_i= \{C_1, C_2,\cdots, C_k\}$ refer to a possible partition of the entity resolution result, where $C_i=\left\{r^{(i)}_1,r^{(i)}_2,\cdots, r^{(i)}_{|C_i|}\right\}$ denotes a cluster of records that represent a unique entity. 
\end{definition}

\begin{remark}
For example, as shown in Table~\ref{tab:partitions}, $P_1, P_2, P_3, P_4$ are four possible partitions. It is important to note that multiple records could refer to the same entity, which is common in large datasets. \end{remark}

\begin{definition}[Result Set]
Let all the possible partitions compose the result set $R$, together with a probability assignment function $\mathcal{P}: P_i \rightarrow [0,1]$. Each partition $P_i \in R$ has the probability $\mathcal{P}(P_i)$ to be correct, and we have $\sum_{P_i \in RS}\mathcal{P}(P_i) = 1$.
\end{definition}

\begin{definition}[Uncertainty of Result]

We employ the Shannon entropy to measure the uncertainty of the result set $R$. Shannon entropy provides a quantitative measure of uncertainty or randomness within data distributions, which can adapt to different data types. It is a domain-independent concept that can be applied across various industries and domains without significant modifications. Also, the calculations of Shannon entropy can be computationally efficient, particularly for tasks involving large datasets. This makes applying entropy-based methods to real-world scenarios involving substantial amounts of data feasible. In our paper, the entropy of $R$ can be defined as:

\begin{equation}
\label{eqrs}
H(R) = -\sum_{P_i\in R} \mathcal{P}(P_i)\log \mathcal{P}(P_i).
\end{equation}
\end{definition}

\begin{definition}[Matching Pair]
A matching pair can be defined as $m_{ij} = (r_i, r_j)$, which indicates a possible matching between $r_i$ and $r_j$. Note that $m_{ij}$ can appear in more than one possible partition, so for any $m_{ij} $, let $\mathcal{P}(m_{ij})$ be the probability of $m_{ij}$ being in the correct partition, then:
\begin{equation}
\mathcal{P}(m_{ij}) = \mathop{\sum_{P_i\in R \atop m_{ij}\in P_i}}\mathcal{P}(P_i).
\end{equation}

As a simple extension, for a set of matches $U = \{m_{ij}^{(1)},\cdots,m_{ij}^{(k)}\}$, let $\mathcal{P}(U)$ be the probability that all matches of $U$ are in the correct partition, then:
\begin{equation}\label{eqai}
\mathcal{P}(U) = \mathop{\sum_{P_i\in R \atop U\subseteq P_i}}\mathcal{P}(P_i).
\end{equation}

Thus, we can define the matching set $MS = \bigcup_{P_i \in R} \{ m_{ij} \mid m_{ij} \in P_i \}$, which contains all the matching pairs in all the possible partitions $P_i$.
\end{definition}

\begin{definition}[Matching Question]
A matching question (MQ) asks whether a match of entities is correct. For example, the $(r_i, r_j)$ refers to $r_i$ and $r_j$ belong to the same person, and the MQ could be asked as ``Given two records $r_i$ and $r_j$, identify whether they refer to the same records and answer me only ``yes’’ or ``no’’.''
\end{definition}

\begin{definition}[Cost of MQ]
To estimate the cost of each question, we have defined a price function $\mathcal{F}: MQ\rightarrow N$. Since the response of LLM is a fixed binary answer, the number of tokens in the response is also constant, so the main variation lies in the MQs.
\end{definition}

\begin{table}[h]
\vspace{-1em}
\centering
\caption{Selected GPT-3.5-Turbo and GPT-4 Series Pricing.}
\label{table:gpt_pricing}
\begin{tabularx}{\columnwidth}{lXX}
\hline
\textbf{Model}&\textbf{Input}&\textbf{Output}\\ 
\hline
GPT-3.5-turbo-0125& \$1.50/1M tokens & \$2.00/1M tokens\\
GPT-3.5-turbo-instruct& \$0.50/1M tokens & \$2.00/1M tokens\\
GPT-4& \$30.00/1M tokens& \$60.00/1M tokens\\
GPT-4-turbo& \$10.00/1M tokens& \$30.00/1M tokens\\
GPT-4o& \$5.00/1M tokens& \$15.00/1M tokens\\
\hline
\end{tabularx}
\end{table}

\begin{table*}[t]
\centering
\caption{The database for this web application stores profiles of these professionals, and each profile contains information like Name, Email, Job Title, Company, and Location.}
\label{tabduplicates}
\begin{tabular}{|c|c|c|c|c|c|}
\hline
ID & Name          & Email               & Title             & Company       & Location       \\
\hline
$r_1$ & John Doe      & johndoe@email.com   & Software Engineer & TechCorp      & San Francisco  \\
$r_2$ & Andy Doe      & andydoe@email.com   & Software Engineer & TechCorp LLC  & SF, CA         \\
$r_3$ & Jane Smith    & janesmith@email.com & Project Manager   & Innovate Tech & New York       \\
$r_4$ & Jane S.       & janes@email.com     & PM                & Innovate Tech & New York             \\
$r_5$ & David Doe     & davidd@email.com    & Developer         & TechCorp      & SF             \\
$r_6$ & J. Smith      & janesmith@email.com & Proj Manager      & Innovate Tech & New York, NY   \\
$r_7$ & John D.       & johndoe@email.com   & Software Eng.     & TechCorp      & San Fran       \\
$r_8$ & Jane A. Smith & janesmith@email.com & Project Mgr       & Innovate      & NYC            \\
$r_9$ & Jonathan Doe  & johndoe2@email.com  & Software Engineer & TechCorp      & San Francisco  \\
$r_{10}$ & Andy Smith   & andysmith@email.com & Manager           & Innovate Tech & New York, NY   \\
$r_{11}$ & David D.     & david@email.com    & Developer & TechCorp      & San Francisco  \\
\hline
\end{tabular}
\end{table*}

\begin{remark}
OpenAI's LLMs like GPT-4 charge for API requests based on the total number of tokens in both the input question and the model's response \footnote{https://openai.com/pricing.}. For example, if an API request is submitted with an input consisting of 20 tokens and receives an output containing two tokens, the user will be billed for 22 tokens. A token refers to the smallest unit of text that the model reads and processes. Tokens can vary in size and represent individual words, subwords, or even characters. Depending on the specific architecture of the LLM, a token could be as short as a single character or as long as an entire word or phrase. 
\end{remark}

Our goal is to maximize uncertainty reduction by asking valuable questions to LLMs. However, due to budget limitations, choosing the correct set of questions becomes the core challenge in this process. In the following section, we describe how our method works.

\section{Uncertainty Reduction Framework}

In this section, we present our uncertainty reduction framework for enhancing entity resolution results. There are three critical components of our framework, namely, (1) Probability Distribution Initialization, (2) Matching Questions Selection, and (3) Adjustment with LLMs' response. We use the Shannon entropy to estimate the uncertainty of the possible partitions, select the most valuable MQs to leverage LLM to verify efficiently, and then adjust the probability distribution to reduce the uncertainty, thus enhancing the final results. The rationale behind our framework is straightforward: within a fixed system, uncertainty (entropy) can only be reduced by external energy input. This is where the LLMs play the role, utilizing their impressive semantic understanding capabilities to provide precise predictions. However, as previously mentioned, the high API cost makes submitting all MQs for LLM verification impractical. Therefore, we designed this uncertainty reduction framework by verifying only a few valuable MQs to achieve the final results efficiently. In the following subsections, we provide detailed descriptions of each procedure.

\subsection{Probability Distribution of Possible Parition}

\begin{table}[t]
    \centering
    \caption{Distribution of Probabilities for Possible Partitions in Table~\ref{tabduplicates}. The sum of all these probabilities is equal to 1.}
    \label{tab:partitions}
    \begin{tabularx}{0.97\columnwidth}{Xlc}
        \hline
        Possible Partition & Probability \\
        \hline
        \(P_1 = \{(r_1,r_7), (r_3,r_4)\}\) & 0.10 \\
        \(P_2 = \{(r_1,r_7,r_9), (r_3,r_4)\}\) & 0.26 \\
        \(P_3 = \{(r_1,r_7,r_9), (r_3,r_4,r_8),(r_5,r_{11})\}\) & 0.36 \\
        \(P_4 = \{(r_1,r_2,r_7,r_9), (r_3,r_4,r_6,r_8,r_{10}),(r_5,r_{11})\}\) & 0.28 \\
        \hline
    \end{tabularx}
\end{table}

Given two records $r_i = \{attr_1^{(i)}, attr_2^{(i)}, \cdots, attr_n^{(i)}\}$  and $r_j = \{attr_1^{(j)}, attr_2^{(j)}, \cdots, attr_n^{(j)}\}$, where $attr_e$ denotes a attribute of the record, continue with the example Table~\ref{tabduplicates}, the attributes can be ``Name'', ``Email'', or any other. We first use any similarity function, such as Edit distance ``Levenshtein'' or Jaro distance, to compute the similarity between corresponding attributes. Then, for any two records $r_i$ and $r_j$ that satisfy the following condition:
\begin{equation}
\forall e \in [1, n], \, \text{Sim}(r_i[\text{attr}_e], r_j[\text{attr}_e]) \geq \tau,
\end{equation}
where $Sim(r_i[attr_e], r_j[attr_e])$ denotes the similarity of each attribute, and $\tau$ denotes a predefined threshold, we consider them a matching pair $m_{ij}$ in a possible partition $P_\tau$. Since there may be more than two records referring to a single entity. We leverage the transitivity property to merge them to a cluster $C_i$. By iteratively implementing this process for every two records with different threshold $\tau$, we can obtain a set of possible partitions $R = \{P_{{\tau}^1}, P_{{\tau}^2},\cdots, P_{{\tau}^n}\}$, i.e., the result set. This can be seen as the initial entity resolution results produced by traditional similarity-based methods, which are often mediocre. 
 
Then, with a probability assignment function $\mathcal{P}(\cdot)$, there are two methods to initialize the probability distribution of possible partitions: 
(1) The probability of each possible partition can be initialized by resembling a normal distribution, where probabilities are higher in the middle and lower at the edges. Specifically, a possible partition obtained by filtering out matching pairs with a higher threshold tends to have higher precision. In comparison, a possible partition with a lower threshold tends to have higher recall. Specifically, we create $n$ samples from a standard normal distribution and then normalize the sorted samples to construct a valid probability distribution;
(2) A simple but effective way is to initialize the probability of each possible partition equally, that is, $1/n$. When the probability distribution is uniform, each event has an equal probability of occurring. This indicates that the entropy is maximized in a uniform distribution, representing the highest level of uncertainty because there is no preference or bias toward any particular outcome. In such a scenario, the information required to describe the distribution is maximized, as all possible outcomes are accounted for equally. 

After initializing the probability distribution of the possible partitions, we delve into the formalization of our framework and illustrate how to compute the expected uncertainty reduction caused by $k$ MQs.

\subsection{Expected Uncertainty Reduction}
To obtain a cost-efficient approach to reducing the uncertainty of $R$, we first need to calculate the expected uncertainty reduction through MQs. For a given set of $k$ MQs, denoted as $S_Q = \left\{MQ_1, MQ_2, \cdots, MQ_k\right\}$, our objective is to deduce the expected decrease in uncertainty resulting from the aggregation of answers to these $k$ MQs. To facilitate this calculation, we introduce two key components: $D_A$, represents the domain of possible answers, and $\mathcal{P}_A$, represents the probability distribution of these answers:
\begin{equation}
\begin{split}
&D_A = \left\{a_i|a_i=\{A_1^{(i)},A_2^{(i)},\cdots,A_k^{(i)}\},\  A_j^{(i)} = Y \ or \  N\right\}, \\
&\mathcal{P}_A = \left\{ \mathcal{P}(a_1), \mathcal{P}(a_2), \cdots, \mathcal{P}(a_{2^k})\right\}.
\end{split}
\end{equation}

Within $D_A$, $A_k$ denotes the answer associated with the given MQ. Since LLMs are required to provide either a binary ``yes’’ or ``no’’ answer, there are two potential states for $A_k$. Moreover, each element $a_i$ corresponds to a potential set of answers for the $k$ MQs, resulting in $|D_A| = 2^k$ distinct combinations, each with an associated probability in $\mathcal{P}_A$. Now let $\Delta H_{S_Q}$ be the amount of uncertainty reduced in result set $R$ by $S_Q$, this can be expressed as follows:
\begin{equation}
\label{eqh}
\begin{split}
\Delta H_{S_Q} &= H(R) - H(R|A_1,A_2,\cdots,A_k) \\
&= H(R) - \left( -\sum_{a_i\in D_A}\mathcal{P}(a_i) \sum_{P_j\in RS}\mathcal{P}(P_j|a_i) \log \mathcal{P}(P_j|a_i)\right).
\end{split}
\end{equation}
 
By employing the Naive Bayes to the conditional probability, we can calculate the reduction in uncertainty attributed to $S_Q$ as follows:
 \begin{equation}
\label{eqeh}
\begin{split}
\Delta H_{S_Q} &= H(R) + \mathop{\sum_{P_j \in R \atop a_i \in D_A}}\mathcal{P}(P_j)\mathcal{P}(a_i|P_j)\log \frac{\mathcal{P}(P_j)\mathcal{P}(a_i|P_j)}{\mathcal{P}(a_i)} \\
&= H(R) + \mathop{\sum_{P_j \in R \atop a_i \in D_A}}\left\{ \mathcal{P}(P_j)\mathcal{P}(a_i|P_j)\log \mathcal{P}(P_j)\mathcal{P}(a_i|P_j)\right. \\
&\qquad \qquad \qquad \quad- \left. \mathcal{P}(P_j)\mathcal{P}(a_i|P_j)\log \mathcal{P}(a_i) \right\}.
 \end{split}
 \end{equation}

Therefore, the expected uncertainty reduction caused by $S_Q$ can be computed by Eq.~\ref{eqeh}, provided that we know the values of the following parameters: $\mathcal{P}(a_i), \mathcal{P}(a_i|P_j)$. Here, we give the method for computing these values.
 
\textbf{Computation of} $\mathcal{P}(a_i)$: $\mathcal{P}(a_i)$ represents the probability that all the matching pairs by MQs associated with $a_i$ are within the correct partition. This can be calculated using Eq.~\ref{eqai}, yielding:
\begin{equation}
\mathcal{P}(a_i) = \sum_{\substack{P_j \in RS \\ t = 1,\cdots,k \\ \forall m^{(i)}_t = Y, m_t \in P_j \\ \forall m^{(i)}_t = N, m_t \notin P_j}} \mathcal{P}(P_j).
\end{equation}

\textbf{Computation of} $\mathcal{P}(a_i|P_j)$: $\mathcal{P}(a_i|P_j)$ is either 1 or 0, depending on whether the correct matching pairs by MQs associated with $a_i$ are all within the partition $P_j$, as well as whether the incorrect matching pairs by MQs associated with $a_i$ are all outside the partition $P_j$. This can be expressed as follows:
\begin{equation}
\mathcal{P}(a_i|P_j) =  \left\{ 
\begin{array}{l}
1, \quad \forall m^{(i)}_t = Y, m_t \in P_j  \\ \quad AND \ \   \forall m^{(i)}_t = N, m_t \notin P_j; \\
0, \quad \exists m^{(i)}_t = Y, m_t \notin P_j \\ \quad OR \ \  \exists m^{(i)}_t = N, m_t \in P_j.
\end{array}
\right.
\end{equation}

After being equipped with the Computation of the above parameters, Eq.~\ref{eqeh} can be streamlined as follows:

\begin{equation}\label{simplification}
\begin{split}
\Delta H_{S_Q} &= H(R) + \sum_{a_i\in D_A} \sum_{\substack{P_j \in R \\ t = 1,\cdots,k \\ \forall m^{(i)}_t = Y, m_t \in P_j \\ \forall m^{(i)}_t = N, m_t \notin P_j}} \mathcal{P}(P_j)\log \mathcal{P}(P_j) \\
&\qquad \quad - \sum_{a_i\in D_A} \Bigg[\sum_{\substack{P_j \in R \\ t = 1,\cdots,k \\ \forall m^{(i)}_t = Y, m_t \in P_j \\ \forall m^{(i)}_t = N, m_t \notin P_j}}\mathcal{P}(P_j)\Bigg]\log\mathcal{P}(a_i) \\
& = H(R) + \sum_{P_j\in R} \mathcal{P}(P_j)\log \mathcal{P}(P_j) -\sum_{a_i\in D_A} \mathcal{P}(a_i) \log\mathcal{P}(a_i).
\end{split}
\end{equation}

Finally, by Eq.~\ref{eqrs}, the first two terms of the Eq.~\ref{simplification} have canceled each other out. So we have:
\begin{equation}
\begin{split}
\Delta H_{S_Q} &= -\sum_{a_i \in D_A} \mathcal{P}(a_i)\log \mathcal{P}(a_i) \\
&= H(D_A).  
\end{split}
\end{equation}

Therefore, our problem of maximizing the uncertainty reduction of $k$ MQs turns out to be maximizing the joint entropy of $D_A$; that is, we need to choose the matching pairs combined with the highest uncertainty and convert them into certainties. 

\subsection{Strategy for MQs Selection}
After establishing the definition of expected uncertainty reduction through $k$ MQs, we can formally define the optimization function at the $n^{th}$ iteration as follows:
\begin{equation}
X:= argmax_{Q_X}\Delta H_{S^{n-1}_Q\cap\{Q_X\}}.
\end{equation}

Here, let $A^{(n-1)}$ represents the answers for $S^{n-1}_Q$. Applying the chain rule of conditional entropy, we derive:
\begin{equation}
H(D_{A^{(n-1)}},A_n) = H(D_{A^{(n-1)}})+H(A_k|D_{A^{(n-1)}}).
\end{equation}

Therefore, our objective at each iteration simplifies to the maximization of conditional entropy:
\begin{equation}
X: = argmax_{A_n} H(A_n|D_{A^{(n-1)}}),
\end{equation}
and: 
\begin{equation}\label{eqiteration}
\begin{split}
&H(A_n|D_{A^{(n-1)}})\\
&= - \sum_{a_i \in D_A^{n-1}} \mathcal{P}(a_i) \Big[ 
 \mathcal{P}(A_n = Y|a_i) \log \mathcal{P}(A_n = Y|a_i) \\
&\qquad \qquad \qquad+ \mathcal{P}(A_n = N|a_i) \log \mathcal{P}(A_n = N|a_i) \Big].
\end{split}
\end{equation}

Eq.~\ref{eqiteration} proves that, at each iteration, we only need to search for the most uncertain MQs, given the MQs selected in previous iterations. Next, to further consider the cost of MQs, we discuss how to efficiently and effectively enable the uncertainty process. Though the API request cost varies depending on the specific LLMs and platforms, we use a price function $\mathcal{F}: MQ \rightarrow N$ to convert the cost to a constant number in our study. Therefore, we can calculate the anticipated uncertainty reduction and the corresponding cost for a given set of $k$ MQs. 

Now, let us consider two situations: 

\textbf{($k=1$)}. We choose the single most valuable MQ for LLM verification at each iteration. This means maximizing the ``bang-per-buck'': the expected uncertainty reduction divided by the cost. Based on the response of the MQ, we then adjust the probability distribution and pose a new MQ for the next iteration. This process continues until the uncertainty reduction stops or we run out of budget. 

\textbf{($k>1$)}. However, there might be numerous matching pairs across all possible partitions in practice. To provide a faster solution and thus speed up the entire process, we can ask $k$ multiple MQs at once. Based on the aggregation of their answers, we can adjust the probability distribution at each iteration. We named this the Matching Questions Selection Problem (MQsSP) and then formalized it. In MQsSP, each MQ can be in one of two states: (1) chosen or (2) awaiting selection. This potentially results in many selection choices, numbering up to $C^k_{|MS|}$. This constitutes a compelling and valuable optimization problem warranting investigation.

\begin{theorem}\label{nphard}
The MQsSP is NP-hard.
\end{theorem}
\begin{proof}
To establish the NP-hardness, we can reduce the Maximum Coverage Problem (MCP) to our problem. An instance of MCP encompasses a universe of elements $U$, a collection $S = {S_1, S_2, \ldots, S_m}$ of subsets of $U$, and an integer $k$. The goal of MCP is to find a sub-collection $S' \subseteq S$ of size at most $k$ that maximizes the number of covered elements, i.e., $|\bigcup_{S \in S'} S|$. It is known that finding the set $S'$ that maximizes the number of covered elements is NP-hard.

We show that for any instance $(U, S)$ of MCP, we can create a corresponding instance of our problem based on $(U, S)$ in polynomial time. We translate the set $U$ of universal items into the set $MS$ of matching pairs to be determined by corresponding MQs. We need to select up to $k$ MQs (each corresponding to a subset in $S$) to maximize the uncertainty reduction at each iteration, which is equivalent to the joint entropy of the answer set $D(A)$. Consider a special case where each MQ cost equals one and the knapsack capacity $B = k$. Finally, since the MCP is NP-hard, and we have shown a polynomial-time reduction from it to our MQsSP, the latter problem must also be NP-hard. Thus, completing the proof.
\end{proof}

\begin{algorithm}[t]
\caption{Greedy-based Algorithm for MQsSP.}
\label{Al}
\begin{algorithmic}[1]
\Require the result set $R$, the price function $\mathcal{F}$, and the remaining budget $B$.
\Ensure A set of $k$ MQs - $S_Q$ maximizing $ H_{S_Q}$.

\State $S_{Q_1} \leftarrow argmax_{S_Q \in N, \mathcal{F}(S_Q) \leq B, |S_Q| < d} H(S_Q)$
\State$S_{Q_2} \leftarrow \emptyset$
\For{all $S_Q \subseteq MS, |S_Q| = d, \mathcal{F}(S_Q) \leq B$}
    \State $MS' \leftarrow MS \setminus S_Q$
    \State $S_{Q_g} \leftarrow S_Q$
    \While{$MS' \neq \emptyset$ and $\mathcal{F}(S_Q) < B$}
        \State Select $m_{ij}^* = \arg\max_{m_{ij} \in MS'} \left\{\frac{H(\{MQ\} \cup S_Q) - H(S_Q)}{\mathcal{F}(MQ)}\right\}$
        \If{$\mathcal{F}(S_{Q_g}) \cup \{MQ\}) \leq B$}
            \State $S_{Q_g} \leftarrow S_{Q_g} \cup \{MQ\}$
        \EndIf
        \State $MS' \leftarrow MS' \setminus \{m_{ij}^*\}$
    \EndWhile
        \If{$H(S_{Q_g}) \geq H(S_{Q_2})$}
        \State $S_{Q_2} \leftarrow S_{Q_g}$
    \EndIf
\EndFor
\State \Return $\arg\max \{H(S_{Q_1}), H(S_{Q_2})\}$
\end{algorithmic}
\end{algorithm}

To efficiently address the MQsSP, we propose a greedy-based algorithm. Maximizing the $H(D_A)$ can be regarded as maximizing the joint entropy of a set of random variables, which has been established as a monotone submodular function. Thus, the basic idea is to maximize the marginal contribution divided by the cost at each step until selecting the $k^{th}$ MQs or running out of budget. When adding an item would violate the budget constraint, the item is discarded, and the iteration continues to inspect possibly cheaper items.

However, as discussed in \cite{horel2015notes, khuller1999budgeted}, simply iteratively selecting the most valuable MQs has an unbounded approximation ratio, especially when there are high-value items. Fortunately, this can be alleviated by partial enumeration. The detailed steps of the algorithm are shown in Algorithm~\ref{Al}. It improves the greedy algorithm by first finding a subset $S_1$ of size $d$ with maximum value, then running the greedy algorithm on all subsets of size $d$. The best result obtained through this process is returned. Theoretically, the approximation ratio of Algorithm 4 is at least $1 - \frac{1}{\sqrt{e}}$.

\subsection{Adjustment with LLMs Response}
Upon receiving the answers to selected MQs from LLM, we will continue with the third part of our uncertainty reduction framework. That is, adjust the probability distribution of possible partitions. In real applications, the LLMs-generated answers may contain mistakes; we must allow for the possibility that any LLMs-generated answers could be wrong. It is also a worthy investigative topic. A simple way to estimate the probability of this happening is to calculate the error rate of the LLMs on ER tasks, and the capability of LLM can be denoted by $\mathcal{P}(\Theta)$. 

\textbf{Running Example}: Continuing with the example in Table~\ref{tab:partitions}, assume the matching pair $m_{48}$ is recognized as correct by the LLM and $\mathcal{P}(\Theta) = 90\%$, we can deduce:
\begin{equation}
\begin{split}
\mathcal{P}&(P_3|e: m_{48} \text{ is answered from LLM}) =\frac{\mathcal{P}(P_3)\mathcal{P}(e|P_3)}{\mathcal{P}(e)}\\
=\ & \frac{\mathcal{P}(P_3)\mathcal{P}(\Theta)}{\mathcal{P}(m_{48})\mathcal{P}(\Theta) + (1- \mathcal{P}(m_{48}))(1-\mathcal{P}(\Theta))} \\
=\ &\frac{0.36 * 0.9}{0.64*0.9+0.36*0.1}\\
=\ &0.53.
\end{split}
\end{equation}
Similarly, for $P_4$, we have: 
\begin{equation}
\begin{split}
\mathcal{P}&(P_4|e: m_{48} \text{ is answered from LLM})) = \frac{\mathcal{P}(P_4)\mathcal{P}(e|P_4)}{\mathcal{P}(e)}\\
= \ & \frac{\mathcal{P}(P_4)\mathcal{P}(\Theta)}{\mathcal{P}(m_{48})\mathcal{P}(\Theta) + (1- \mathcal{P}(m_{48}))(1-\mathcal{P}(\Theta))} \\
= \ &\frac{0.28 * 0.9}{0.64*0.9+0.36*0.1}\\
= \  &0.41.
\end{split}
\end{equation}
And for $P_1,P_2$ that do not contain the $m_{48}$ as a matching pair:
\begin{equation}
\begin{split}
\mathcal{P}&(P_1|e: m_{48} \text{ is answered from LLM})) = \frac{\mathcal{P}(P_1)\mathcal{P}(e|P_1)}{\mathcal{P}(e)}\\
= \ & \frac{\mathcal{P}(P_1)(1-\mathcal{P}(\Theta))}{\mathcal{P}(m_{48})\mathcal{P}(\Theta) + (1- \mathcal{P}(m_{48}))(1-\mathcal{P}(\Theta))} \\
= \ &\frac{0.10 * 0.1}{0.64*0.9+0.36*0.1}\\
= \  & 0.02,
\end{split}
\end{equation}
And:
\begin{equation}
\begin{split}
\mathcal{P}&(P_2|e: m_{48} \text{ is answered from LLM})) = \frac{\mathcal{P}(P_2)\mathcal{P}(e|P_2)}{\mathcal{P}(e)}\\
= \ & \frac{\mathcal{P}(P_1)(1-\mathcal{P}(\Theta))}{\mathcal{P}(m_{48})\mathcal{P}(\Theta) + (1- \mathcal{P}(m_{48}))(1-\mathcal{P}(\Theta))} \\
= \ &\frac{0.26 * 0.1}{0.64*0.9+0.36*0.1}\\
= \  & 0.04.
\end{split}
\end{equation}

Thus, the uncertainty of the possible partitions is reduced from 2 to 1. By our error-tolerant technique, the decrease is less than if the answers were obtained with no error. In short, even imperfect answers can help reduce the uncertainty. Moreover, for the cases when $k>1$, it is easy to perform the algebraic manipulations to show that, for any two answers $e_1$ and $e_2$, we have:
\begin{equation}\label{independent}
\mathcal{P}(P_i|e_1,e_2) = \mathcal{P}(P_i|e_2,e_1),
\end{equation}
which indicates that the final result of adjustment is independent of the sequence of the answers. In other words, when we have a set of MQs, it does not matter in what sequence the answers are used for adjustment.

However, the correct answer may exist outside the initial possible partitions. Therefore, we introduce a new approach to reach the correct partition by dynamically deleting matching pairs that are certain to be wrong by LLMs. Let us consider a special case where most matching pairs are correct in a possible partition but contain a few error pairs because our initialization does not contain all the possible partitions up to $2^n$ possible partitions. The best way is to generate the correct partition dynamically. Technically, we first adjust the probability via Naive Bayes, as mentioned above. We will improve the probability of possible partitions that do not contain this matching pair and reduce the probability of possible partitions that contain this matching pair. Then, we remove the matching pair from the possible partitions which contain it. If the possible partitions after removing the error-matching pair still contain most of the correct matches, the probability of it being correct will be increased rapidly in the next iteration.

\section{Experiments}

In this section,  we evaluate our methods and report experimental results. We focus on evaluating two issues. First, we examine the effectiveness of our proposed method in reducing the uncertainty for possible partitions. Second, we verify the cost-efficiency of our algorithm.

\subsection{Experimental Setup}

\begin{table}[t]
\centering
\caption{Datasets Information}
\label{dataset}
\begin{tabularx}{0.97\columnwidth}{|c|c|c|c|c|}
\hline
\textbf{Dataset}&\textbf{Domain} &\textbf{Attr.} & \textbf{Pairs} & \textbf{Matches}\\
\hline
DBLP-ACM&Citation&4&12,363&2,220\\                          
\hline
Walmart-Amazon&Electronics&5&10,242&962\\       
\hline
Abt-Buy&Product&3&9,575&1,028\\
\hline
\end{tabularx}
\end{table}

\begin{figure*}[t]
\includegraphics[width=0.98\textwidth]{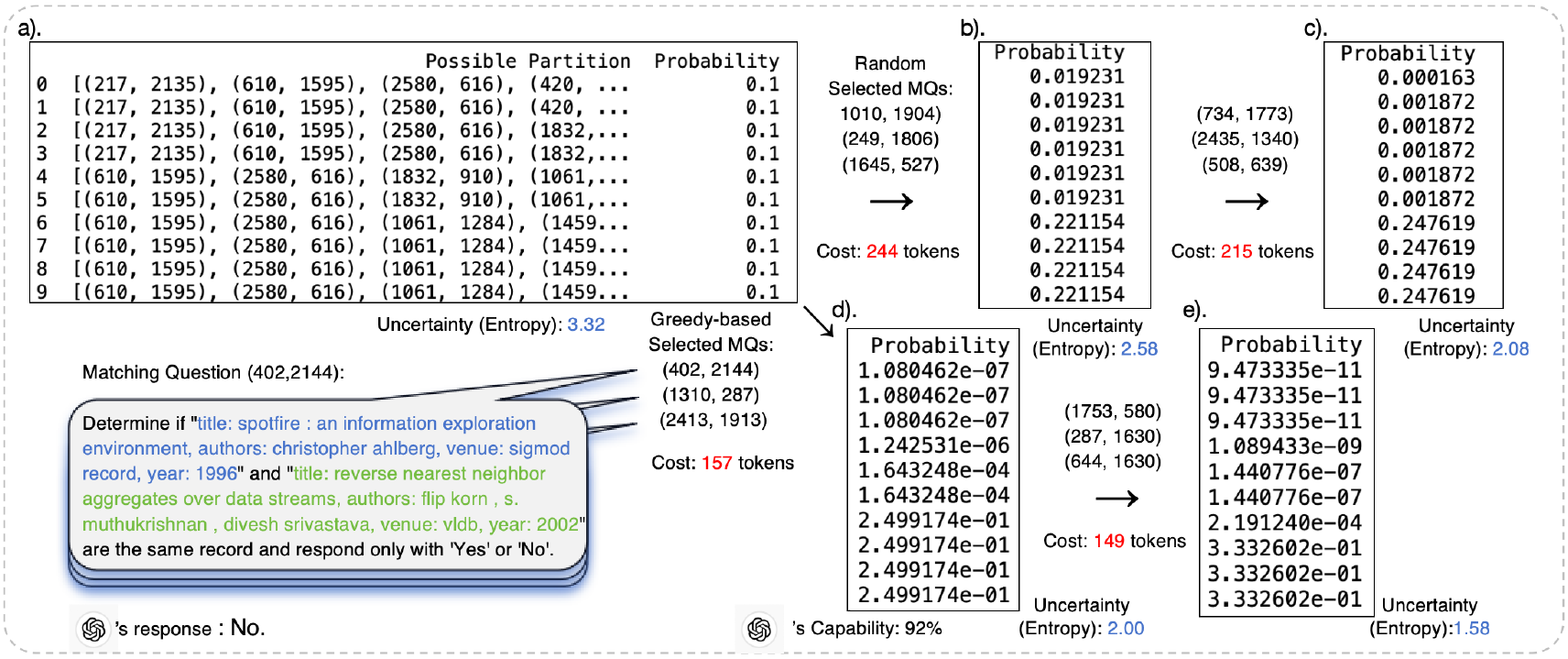}
\caption{A real case of our uncertainty reduction framework: a). shows the initial probability distribution of possible partitions; b), c). shows the probability distribution after verifying the MQs adjusted by random selection; d), e). shows the probability distribution after verifying the MQs adjusted by greedy selection. We annotate the uncertainty of each state and the cost used in each step.}
\label{cases}
\end{figure*}

\textbf{Datasets.} We use three real-world datasets widely adopted by existing works \cite{fan2024cost, li2020deep}. (1) DBLP-ACM is a dataset containing citation records from DBLP and ACM with four attributes: Title, Year, Author, and Venue. (2) Walmart-Amazon dataset contains information on electronics over two shopping platforms. The datasets have five attributes: Title, Category, Brand, model, and Price. (3) Abt-Buy is a dataset that contains product information with three attributes: Name, Description, and Price. The information of these three datasets is introduced in Table~\ref{dataset}. Note that there may be no data in some attributes.\\
\textbf{Similarity Functions.} We use three similarity functions, ``levenshtein distance'', ``jaro distance'', and ``jaccard distance'' \cite{LevenshteinDistance, jaro1989advances}. All of them are implemented by the Record Linkage  \cite{de_bruin_j_2019_3559043}, a widely-used Python library. To accommodate diverse matching scenarios, we implemented a series of thresholds ranging from 0.50 to 0.95, with increments of 0.05. This approach generated a spectrum of possible partitions.\\
\textbf{Evaluation Metrics.} For quality, we use accuracy, precision, and recall. Suppose the set of pairs that refer to the same entity is $S_T$ and the set of pairs that an algorithm reports as the same entity is $S_P$. Then the precision is $p = \frac{|S_T\cap S_P|}{|S_P|}$, the recall is $r = \frac{|S_T\cup S_P|}{|S_T|}$. For cost efficiency, we compare the uncertainty reduction curve of different question selection strategies on API cost, irrespective of the various LLMs pricing information; we count the number of tokens used in the uncertainty reduction process.

\subsection{Evaluation on LLM}

To study the effectiveness of our uncertainty reduction framework, we first evaluate the capabilities of different LLMs on domain-specific datasets. We access these models through their respective API services. We conduct a preliminary evaluation for each domain using a small test set comprising 100 labeled matching pairs. This serves as a measure of each LLM's capability to handle entity resolution tasks. As illustrated in Figure~\ref{capability}, the GPT-4 series is currently recognized as the most advanced LLM globally, offering unparalleled performance.
In contrast, the GPT-3.5 series, while less capable, is noted for its cost-effectiveness, making it a viable option for budget-conscious applications. Additionally, we explored the capabilities of the Claude 3 model, which has demonstrated robust performance in addressing scientific queries. For this study, we leverage OpenAI's GPT-4-turbo to optimize our entity resolution results, capitalizing on its advanced features and high accuracy in complex data scenarios. 

\begin{figure}[t]
    \centering
    \includegraphics[width=0.45\textwidth]{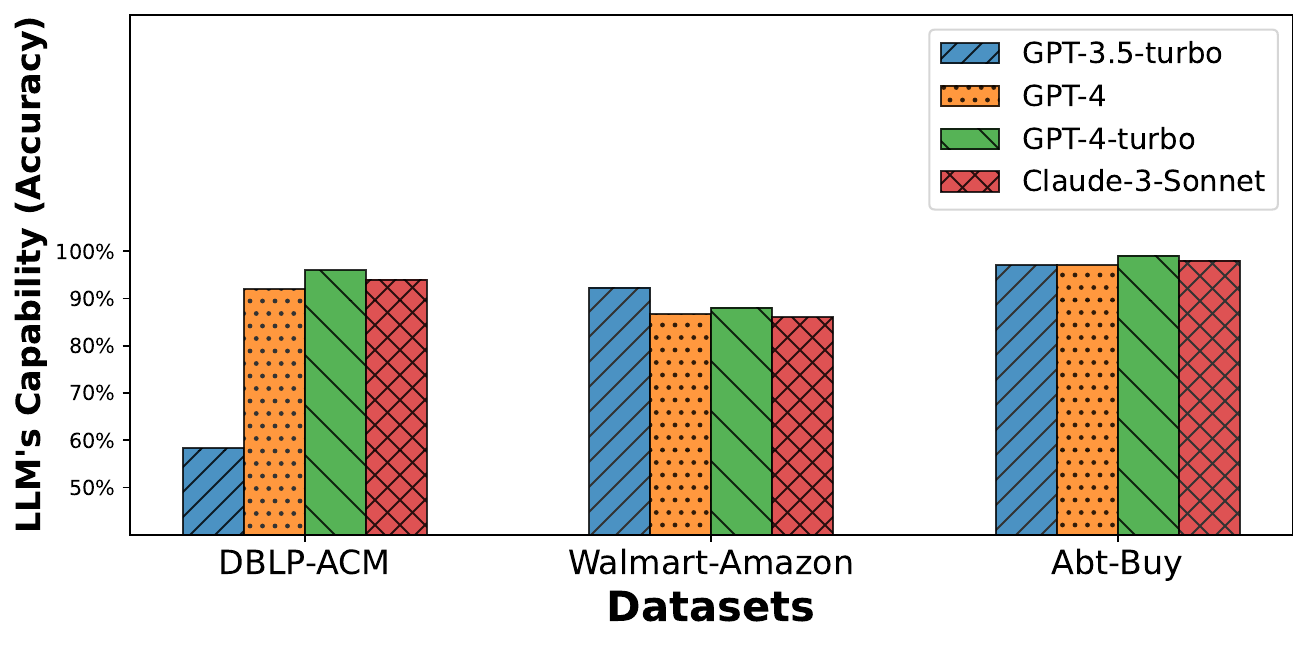}
    \vspace{-1em}
    \caption{The abilities of various LLMs on datasets in different fields. All results are tested three times and averaged.}
    \label{capability}
    \vspace{-2em}
\end{figure}

\begin{figure}[t]
\centering
\subfloat[DBLP-ACM 1k Budget]{
  \includegraphics[width=0.23\textwidth]{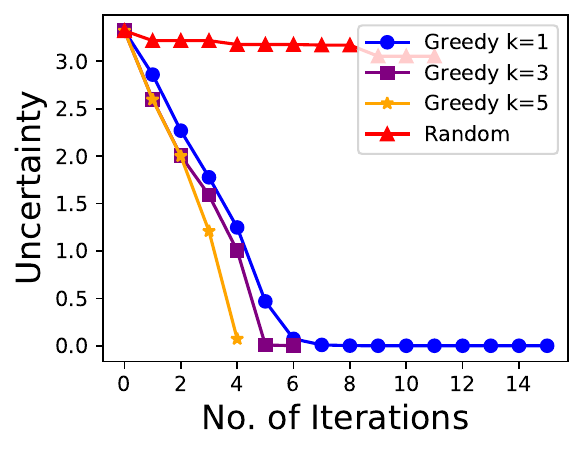}
}
\subfloat[DBLP-ACM 2k Budget]{
  \includegraphics[width=0.23\textwidth]{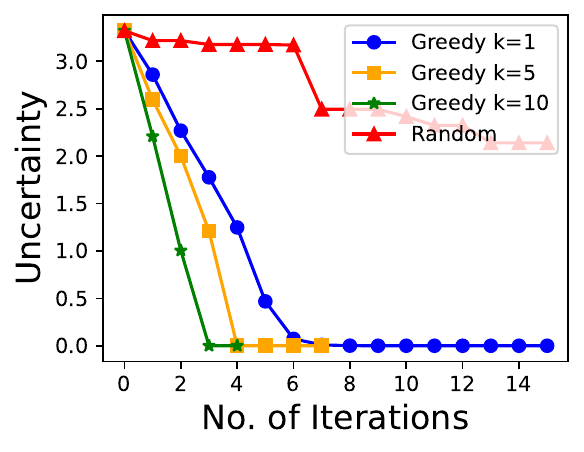}
}
\\
\vspace{0.5em}
\subfloat[Walmart-Amazon 1k Budget]{
  \includegraphics[width=0.23\textwidth]{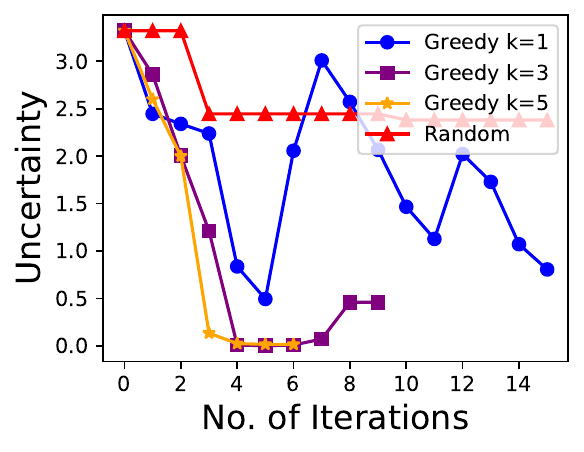}
}
\subfloat[Walmart-Amazon 2k Budget]{
  \includegraphics[width=0.23\textwidth]{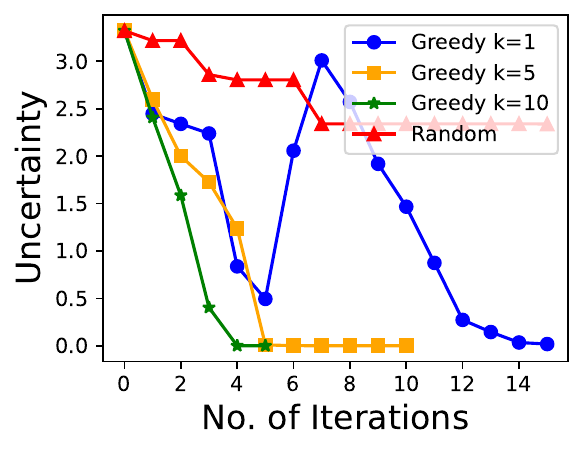}
}
\\
\vspace{0.5em}
\subfloat[Abt-Buy 1k Budget]{
  \includegraphics[width=0.23\textwidth]{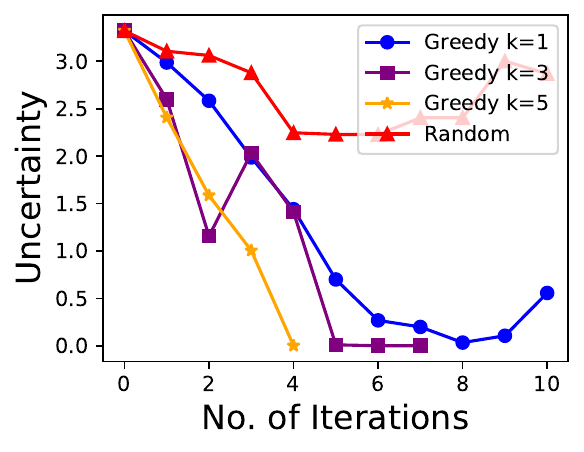}
}
\subfloat[Abt-Buy 2k Budget]{
  \includegraphics[width=0.23\textwidth]{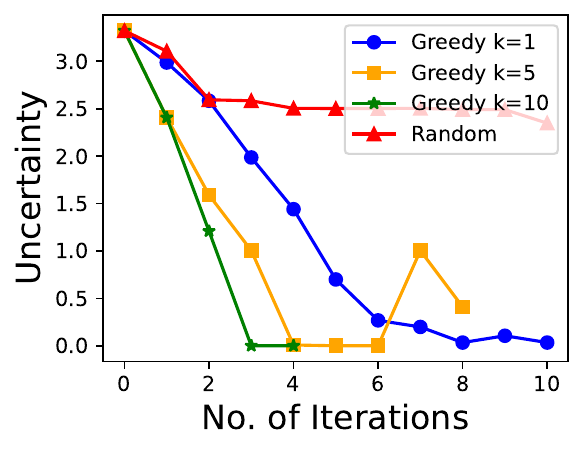}
}
\caption{Random Selection v.s. Greedy Approximation Selection with 1k \& 2k Budget-Constraint.}
\label{exp}
\vspace{-1.5em}
\end{figure}

Then, we present a practical example within the uncertainty reduction process. To illustrate the efficacy of our approximate algorithm, we compared our method against random selection. As depicted in Figure~\ref{cases}, our greedy-based algorithm consistently selects the most valuable yet least costly questions, thereby expediting uncertainty reduction. Conversely, random selection of MQs for LLM verification results in a slower and more costly uncertainty reduction process.

Next, we present a comparative analysis of various $k$ settings in Figure~\ref{exp}. When $k=1$, only the single most valuable matching question is submitted to the LLM for verification. When $k>1$, multiple MQs are submitted simultaneously, accelerating the process as the primary delay arises from awaiting the LLM's response. Each iteration can also be regarded as a discrete-time step in real-world implementations. While results vary by data type and the baseline performance of the LLM on specific datasets, our findings consistently demonstrate that our greedy algorithm significantly outperforms random selection methods. Additionally, we conducted experiments to assess uncertainty reduction in low-budget scenarios. In these tests, we set $k=1,3,5$ with a 1k token budget and  $k=1,5,10$ with a 2k token budget. These conditions revealed that higher numbers of MQs enhance the uncertainty reduction effect. After examining various $k$ settings and budget constraints, we conclude that increasing the number of iterations in limited budget scenarios is more beneficial. Conversely, with a sufficient budget, querying multiple MQs concurrently is preferable to optimize time efficiency.

An interesting observation is that the uncertainty of the result set occasionally increased after certain adjustments. This phenomenon likely arises when high-confidence partitions conflict with subsequent responses from the LLM, thereby altering the potential partitions and increasing uncertainty. Nevertheless, uncertainty is eventually reduced to a minimal level, provided there is a sufficient budget. Furthermore, our experiments elucidate the relationship between budget constraints and the pace of uncertainty reduction. We noted that uncertainty diminishes more rapidly in scenarios with larger budgets, primarily due to the increased number of queries permissible within the same timeframe. However, we also observed that a higher budget does not invariably lead to a proportional decrease in uncertainty.

\subsection{Data Quality}

To demonstrate the correctness of our proposed method, we assess the precision and recall of the most probable partition after each iteration—a process of reducing uncertainty. Figure~\ref{exp2} illustrates the precision and recall across different MQs selection methods per iteration. Notably, our method markedly enhances precision compared to random-based selection. However, we observe a decline in recall as $k$ increases, particularly for $k=5$ and $k=10$. This trend can be attributed to three factors: firstly, the initial possible partition may include too many erroneous matching pairs; secondly, the informational value of selecting $k$ MQs is lower than when $k=1$; thirdly, given the NP-hardness of the problem, our selection of MQs is constrained to those near-optimal.

\section{Related Works}
In recent decades, many researchers have focused on integrating machine learning to handle uncertainty in entity resolution more effectively. Notably, Christen and Goiser \cite{christen2007quality} applied supervised learning to improve the quality of probabilistic record linkage. Ebraheem et al. \cite{ebraheem2017deeper} investigated the use of deep learning to resolve entities under uncertainty, achieving notable improvements in accuracy. Hassanzadeh and Miller \cite{hassanzadeh2009creating} addressed the inherent ambiguities in data sources, proposing methods to clarify these ambiguities in entity resolution processes. Moreover, Bhattacharya and Getoor \cite{bhattacharya2007collective} utilized graphical models to represent and mitigate uncertainty in entity resolution, demonstrating the effectiveness of probabilistic models. These studies collectively advance our understanding of uncertainty management in entity resolution, setting the stage for developing more precise and reliable techniques in complex data environments.

\begin{figure}[t]
\subfloat[Precision]{
  \includegraphics[width=0.23\textwidth]{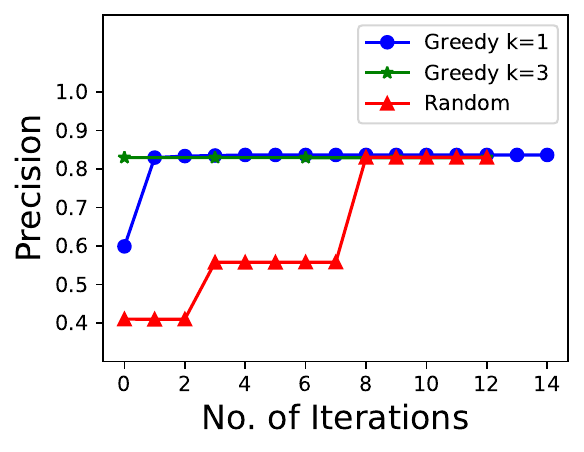}
}
\subfloat[Recall]{
  \includegraphics[width=0.23\textwidth]{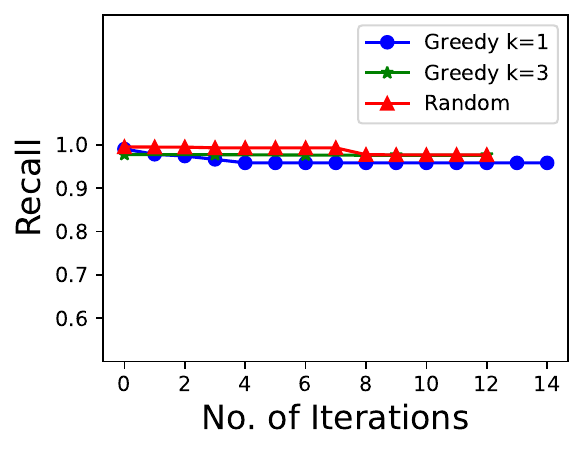}
}
\vspace{-0.5em}
\caption{Data Quality with Budget Constraint.}
\vspace{-1.5em}
\label{exp2}
\end{figure}

The application of language models in entity resolution represents a burgeoning and promising field of research. As the scale of pre-training data and model parameters increases, these large models offer emergent capabilities that make entity resolution tasks effective and accessible. Yuliang Li and Jinfeng Li \cite{Li_2020} have illustrated transformer-based models' automation and enhancement potential, particularly in managing extensive, real-world datasets. Further, Jiawei Tang et al. \cite{tang2022generic} have employed GPT-3 in complex datasets, highlighting its advanced language comprehension as a valuable asset in identifying and clarifying ambiguous entity references. Recent studies \cite{narayan2022can, peeters2023entity} have explored the application of LLMs with minimal supervision from labeled data, revealing that in-context learning significantly boosts the models' performance. However, as prompt lengths increase, so does the cost of API requests. Zhang et al. \cite{zhang2023large} and Fan et al. \cite{fan2024cost} have adopted batch prompt techniques to mitigate these costs. In contrast, our framework focuses on cost reduction by strategically selecting valuable MQs, which can be combined with these methods to decrease expenses further.

\section{Conclusion}

In this study, we introduce an uncertainty reduction framework that uses LLMs to enhance entity resolution. To balance the effectiveness with the API request cost, our main contributions include simplifying the selection process of MQs, which we show to be mathematically equivalent to the joint entropy of possible answer sets. This finding has profound implications for the efficiency of the question selection process. Moreover, to allow a parallel solution, we prove the NP-hardness of the MQsSP and provide a polynomial-time approximation algorithm leveraging submodular properties for near-optimal results. Lastly, we propose an error-tolerant method that significantly improves outcomes. Experimental studies have demonstrated the effectiveness and correctness of our method. Future work includes utilizing LLM's confidence in each answer to achieve a more precise uncertainty reduction process.



\bibliographystyle{ACM-Reference-Format}
\bibliography{sample-base}


\end{document}